\pdfoutput=1

\documentclass[nohyperref]{article}

\usepackage{microtype}
\usepackage{graphicx}
\usepackage{subfigure}
\usepackage{booktabs} 

\usepackage{hyperref}

\usepackage[accepted]{icml2022}

\usepackage{amsmath}
\usepackage{amssymb}
\usepackage{mathtools}
\usepackage{amsthm} 
\usepackage{bbm}

\usepackage[capitalize,noabbrev]{cleveref}

\theoremstyle{plain}
\newtheorem{theorem}{Theorem}[section]

\newtheorem{lemma}[theorem]{Lemma}

\theoremstyle{definition}
\newtheorem{definition}[theorem]{Definition}

\theoremstyle{remark}

\newcommand{\argmax}{\mathop{\arg\max}}
\DeclareMathOperator{\var}{Var}
\DeclareMathOperator{\poly}{poly}

\newcommand{\abs}[1]{\left\lvert#1\right\rvert}
\newcommand{\set}[1]{\left\{#1\right\}}

\usepackage{todonotes}

\icmltitlerunning{Optimal Clustering with Noisy Queries via Multi-Armed Bandit}

\begin{document}

\twocolumn[
\icmltitle{Optimal Clustering with Noisy Queries via Multi-Armed Bandit}



\icmlsetsymbol{equal}{*}

\begin{icmlauthorlist}
\icmlauthor{Jinghui Xia}{fudan}
\icmlauthor{Zengfeng Huang}{fudan,lab}
\end{icmlauthorlist}

\icmlaffiliation{fudan}{School of Data Science, Fudan University, Shanghai, China}
\icmlaffiliation{lab}{Shanghai Key Lab of Intelligent Information Processing}

\icmlcorrespondingauthor{Zengfeng Huang}{huangzf@fudan.edu.cn}

\icmlkeywords{Machine Learning, ICML}

\vskip 0.3in
]



\printAffiliationsAndNotice{}  

\begin{abstract}
 Motivated by many applications, we study clustering with a faulty oracle. In this problem, there are $n$ items belonging to $k$ unknown clusters, and the algorithm is allowed to ask the oracle whether two items belong to the same cluster or not. However, the answer from the oracle is correct only with probability $\frac{1}{2}+\frac{\delta}{2}$. The goal is to recover the hidden clusters with minimum number of noisy queries. Previous works have shown that the problem can be solved with $O(\frac{nk\log n}{\delta^2} + \poly(k,\frac{1}{\delta}, \log n))$ queries, while $\Omega(\frac{nk}{\delta^2})$ queries is known to be necessary. So, for any values of $k$ and $\delta$, there is still a non-trivial gap between upper and lower bounds. In this work, we obtain the first matching upper and lower bounds for a wide range of parameters. In particular, a new polynomial time algorithm with $O(\frac{n(k+\log n)}{\delta^2} + \poly(k,\frac{1}{\delta}, \log n))$ queries is proposed. Moreover, we prove a new lower bound of $\Omega(\frac{n\log n}{\delta^2})$, which, combined with the existing $\Omega(\frac{nk}{\delta^2})$ bound, matches our upper bound up to an additive $\poly(k,\frac{1}{\delta},\log n)$ term. To obtain the new results, our main ingredient is an interesting connection between our problem and multi-armed bandit, which might provide useful insights for other similar problems.
\end{abstract}

\section{Introduction}
In this paper, we study clustering with noisy queries, a.k.a. clustering with a faulty oracle.  In this problem, there are $n$ vertices (or items) belonging to $k$ unknown clusters, and the algorithm is allowed to ask the oracle whether two vertices belong to the same cluster or not. However, the answer from the oracle is correct only with probability $\frac{1}{2}+\frac{\delta}{2}$. The goal is to recover the hidden clusters with minimum number of noisy queries. This elegant theoretical model is first proposed in \cite{mazumdar17a}, which is motivated by many application scenarios, e.g., \emph{crowdsourced entity resolution} (CER). Entity resolution aims to identify all records in a database that refer to the same underlying entity, which is a fundamental data mining task \cite{fellegi1969theory,elmagarmid2006duplicate,getoor2012entity}. Based on crowdsourcing platforms like Amazon Mechanical Turk, leveraging human knowledge to improve the accuracy of automated ER techniques is a promising direction \cite{karger2011iterative, wang2012crowder,dalvi2013aggregating,vesdapunt2014crowdsourcing, mazumdar2017theoretical}. The answers from crowd
workers are often very noisy \cite{prelec2017solution} and the goal of CER is to recover entities with minimum number of queries. Thus, clustering with noisy queries well captures CER. See \cite{mazumdar17a} for discussions on more real applications, e.g., signed edge prediction and correlation clustering.

On the theory side, the model is intimately connected with the stochastic block model (SBM). SBM is a popular random graph model for studying clustering and community detection algorithms, and plays a key role in studying statistical and computational tradeoffs for many statistical problems \cite{mcsherry2001spectral, abbe2017community}. In SBM, $n$ vertices belong to $k$ hidden clusters; each inter-cluster edge exists with probability $q$ and each intra-cluster edge exists with probability $p$, where $0\le q < p \le 1$. Given a graph sampled from the above 
model, the goal is to recover all hidden clusters. When $q=\frac{1}{2}-\frac{\delta}{2}$ and $p=\frac{1}{2}+\frac{\delta}{2}$, SBM is equivalent to clustering with noisy queries, except that SBM always observes all $O(n^2)$ queries while our goal is to use $o(n^2)$ queries. 

\paragraph{Problem definition.} Given a set $V$ of $n$ vertices (or items), there is an unknown partition $\set{V_1,...,V_k}$ of $V$ and each vertex $i$ belongs to a unique cluster $V_{c(i)}$.  Define a function $\tau:V\times V\rightarrow{\pm 1}$ such that $\tau(u,v)=1$ if $u,v$ belong to the same cluster and $\tau(u,v)=-1$ if $u,v$ belong to different clusters. To know whether a pair $(u, v)$ is in the same cluster, the algorithm can make queries to an oracle, but the oracle will only return the correct answer with probability $\frac{1}{2}+\frac{\delta}{2}$ for some $\delta\in(0,1)$. Equivalently, the algorithm has access to a noisy version of $\tau$, denoted as $\tilde{\tau}$:
\begin{align*}
	\tilde{\tau}(u,v) = \sigma_{u,v}\cdot \tau(u,v),
\end{align*}
where $\sigma_{u,v}$ is a $\set{\pm 1}$ random variable attaining $+1$ with probability $\frac{1}{2}+\frac{\delta}{2}$.
We assume that $\sigma_{u,v}$ is independent across different pairs and the oracle always returns the same answer for queries to the same pair. 
The goal is to recover $V_i$, $i=1,...,k$ with high probability by making minimum number of queries to the oracle.

\subsection{Previous Results}
\citeauthor{mazumdar17a} \citeyearpar{mazumdar17a} gave an super-polynomial time algorithm with $O(\frac{nk\log n}{\delta^2})$ queries, which recovers all clusters of size $\Omega(\frac{\log n}{\delta^2})$. Then they presented a polynomial time algorithm using $O(\frac{nk\log n}{\delta^2} + \min\{\frac{nk^2\log n}{\delta^4},\frac{k^5\log^2 n }{\delta^8})\}$ queries, which recovers all clusters of size $\Omega(\frac{k\log n}{\delta^4})$. In the same paper, an information-theoretic lower bound of $\Omega(\frac{nk}{\delta^2})$ queries is proved. Later, \citeauthor{green2020clustering} \citeyearpar{green2020clustering} focused on the two-cluster case, and  proposed an algorithm with query complexity $O(\frac{n\log n}{\delta^2} + \frac{\log^2 n }{\delta^6})$. Very recently, \citeauthor{peng21a} \citeyearpar{peng21a} focused on improving the dependency on $\delta$, and gave a polynomial time algorithm with query complexity $O(\frac{nk\log n}{\delta^2} + \frac{k^{10}\log^2 n }{\delta^4})$, which recovers all clusters of size $\Omega(\frac{k^4\log n}{\delta^2})$. Compared to \cite{mazumdar17a,green2020clustering}, the dependence on $\delta$ in the second term is improved to $1/\delta^4$. Note, the lower bound suggests that the setting is meaningful only if $\frac{1}{\delta}\le n^{1/2
}$, since otherwise $\Omega(n^2)$ queries are required. Thus, $\frac{1}{\delta^4}\le \frac{n}{\delta^2}$, and if we ignore $k$ and $\log n$, the dependence on $\delta$ is optimal. 

However, for any $k$ and $\delta$, there is still a non-trivial gap between upper and lower bounds. In particular, whether the first term $O(\frac{nk\log n}{\delta^2})$ can be improved, even without running time constraint, remains an unanswered question. As this term dominates the complexity for small $k,1/\delta$ (which is likely the case), completely settling the question is critical to understanding the nature of the problem. In this paper, we fully resolve this question.

\subsection{Our Contribution}
We propose a new polynomial time algorithm with $O(\frac{n(k+\log n)}{\delta^2} + \frac{k^{8}\log^3 n}{\delta^4})$ queries, which is the first improvement on the first term since \cite{mazumdar17a}. The dependence on $k$ in the second term is also slightly improved.

\begin{theorem} \label{thm.nc}
    There exists a polynomial time algorithm that recovers all the clusters of size $\Omega(\frac{k^4\log n}{\delta^2})$ with success probability $1-o_n(1)$. The total number of queries to the faulty oracle is $O(\frac{n(k+\log n)}{\delta^2}+\frac{k^{8}\log^3 n}{\delta^4})$.
\end{theorem}

\vspace{-0.1cm}
Since $\Omega(\frac{nk}{\delta^2})$ is a lower bound, our algorithm is optimal when $k=\Omega(\log n)$. But for $k=o(\log n)$, there is still a non-negligible gap, and it is previously unknown whether $\frac{n\log n}{\delta^2}$ is necessary for constant $k$. To this end, we prove a new information-theoretic lower bound for any $k\ge 2$, which matches our upper bound for all $k$.

\begin{theorem}
\label{thm.lb}
    For $k\ge 2$, any (randomized) algorithm must make $\Omega(\frac{n\log n}{\delta^2})$ expected number of queries to recover the correct clusters with probability at least $\frac{7}{8}$.
\end{theorem}

To obtain those improved results, we introduce several new techniques. For the upper bound, we follow the framework of \cite{peng21a}, but make several critical changes. The most important new ingredient is to utilize multi-armed bandit algorithms to identify true clusters of vertices. Our lower bound proof is also quite different from \cite{mazumdar17a}. Specifically, we introduce a new variant of best arm identification, and prove a tight lower bound on the sample complexity; then we give a novel reduction from this problem to clustering with noisy queries.

\subsection{Other Related Work}
Clustering is a central problem in data mining and machine learning; numerous algorithms have been proposed, e.g., \cite{mclachlan1988mixture, karypis1998fast, ng2002spectral, brandes2003experiments, arthur2006k}. However, clustering problems are often computationally hard. With the emergence of crowdsourcing platforms, the humans-in-the-loop approaches become a popular research direction. One line closely related to ours is the semi-supervised active clustering framework \cite{ashtiani2016clustering}. Their goal is to incorporate same-cluster queries in specific clustering problems to alleviate the computational hardness, e.g., k-means \cite{chien2018query,bressan2020exact,li2021learning}, correlation clustering \cite{ailon2018approximate,saha2019correlation}. 

\section{Preliminaries and Tools}
We first introduce two important definitions.
\begin{definition} \label{def:balance}
    Let $b\in[0,1]$ and $V$ be a vertex set. We call a partition $V_1,...,V_k$ of $V$ $b$-balanced, if $\abs{V_i}\ge \frac{bn}{k},\forall i\in[k]$. An input instance $V$ is $b$-balanced if its hidden clusters form a $b$-balanced partition.
\end{definition}

\begin{definition}\label{def:bias}
	Let $\eta\in[0,\frac{1}{2}]$. Let $C=V_i$ be a true cluster of $V$ for some $i\in[k]$. A set of vertices $B\subseteq V$ is called $(\eta$, C)-biased if at least $1/2+\eta$ fraction of the vertices in $B$ belong to $C$, i.e. $|B\cap C|\ge(1/2+\eta)\cdot|B|$. When $\eta =1/2$, i.e., $B$ is a subset of some cluster, $B$ is called a sub-cluster.
\end{definition}

\paragraph{Balanced clustering in stochastic block model.}
Fix a $k$-partition $V_1,...,V_k$ of $V$ and a parameter $\delta\in[0,1)$, the distribution $\mathcal{D}(V_1,...,V_k,\delta)$ samples a random graph as follows: for any two vertices $u,v\in V$, if they are in the same cluster, then there is an edge between them with probability $\frac{1}{2}+\frac{\delta}{2}$, otherwise there is an edge between them with probability $\frac{1}{2}-\frac{\delta}{2}$. This distribution is a special case of the well-studied stochastic block model (SBM). When $\delta$ is not too small, there exist polynomial time algorithms that recover $V_1,...,V_k$ from the random graph $G\sim\mathcal{D}(V_1,...,V_k,\delta)$ with high probability. Here, we will need the following result from \cite{peng21a}, which instantiates a general theorem from \cite{vu2018a}.

\begin{lemma}[\cite{vu2018a}]
	\label{thm.vu}
	Let $\delta\in(0,1]$, $n=|V|$ and $G\sim\mathcal{D}(V_1,...,V_k,\delta)$. Suppose the partition $V_1,...,V_k$ is $b$-balanced for some $b\in(0,1]$. Then there exists an algorithm, denoted by $\mathsf{BalSBM}(G,k,\delta,b)$, that recovers all the clusters $V_1,...,V_k$ of $G$ in polynomial time with probability at least $1-n^{-8}$, if 
	$
	n\ge \frac{c_0k^2\log n}{b^2\delta^2}
	$
	for some large enough constant $c_0$.
\end{lemma}

\paragraph{Best arm identification.} We considered a special case of best arm identification (BAI). In this problem, there are $k$ different arms $a_1,...,a_k$. When the algorithm pulls $a_i$, it receives a reward $r\in\set{0,1}$ from $a_i$. Assume the rewards from each $a_i$ are i.i.d. $\mathsf{Bernoulli}(\mu_i)$, where $\mu_i\in[0,1],\forall i$. The goal is to identify $i^* =\argmax_i \mu_i $ using minimum number of pulls. We often use sample complexity to denote the total number of pulls conducted by an algorithm. In our application of BAI, the input instance always satisfies the following conditions: $i^*\in[k]$ is unique,  $\mu_{i^*} \ge \frac{1}{2}+\frac{\delta}{2}$, and $\mu_i \le \frac{1}{2}-\frac{\delta}{2}$ for all $i\neq i^*$. BAI is well-studied and the following results will be used, which is a special case of Theorem 10 from \cite{evendar06a}.

\begin{lemma}\label{lem:BAI}
	Given $\delta\in(0,1)$, $\alpha\in(0,1)$ and a $k$-BAI instance defined above, there exists an algorithm $\mathsf{ME}(k, \delta,\alpha)$ that identifies the best arm with probability at least $1-\alpha$, and its sample complexity is $O(\frac{k}{\delta^2}\log\frac{1}{\alpha})$. 
\end{lemma}

\section{Algorithm for Nearly Balanced Instances}\label{sec:balancedCase}

In this section, we provide an algorithm for nearly balanced instances (see Definition~\ref{def:balance}), which will be used as a subroutine in the algorithm for general instances. For this easier case, the algorithm is simpler, and the intuition behind using BAI to improve the sample complexity is clearer.

\begin{theorem} \label{thm.bnc}
    Let $b\in(0,1]$ and $n\ge\frac{c_0k^2\log^2 n}{b^2\delta^2}$ for some constant $c_0$. Suppose $V$ is a $b$-balanced instance. There exists a polynomial time algorithm that recovers all the clusters with success probability $1-o_n(1)$. The total number of queries to the faulty oracle is $O(\frac{n(k+\log n)}{\delta^2}+\frac{k^{4}\log^2 n}{b^4\delta^4})$.
\end{theorem}

Our algorithm and the algorithm of \cite{peng21a} both consist of two phases, and their first phase is the same.

 In the first phase, $k$ sub-clusters is recovered with high probability (see Algorithm~\ref{alg.subcluster}). In this algorithm, a subset $T$ of vertices is sampled from $V$ uniformly at random, and then query all pairs of vertices in $T$ using $O(|T|^2)$ queries. A graph $H_T = (T, E_T)$ is constructed, where $E_T$ contains all pairs $(u,v) \in T^2$ such that $u\neq v$ and $\tilde{\tau}(u,v) = 1$. Applying standard concentration inequalities, one can show $T$ is $b/2$-balanced, and $H_T$ satisfies the premises of Lemma~\ref{thm.vu} when $|T| = \Omega(\frac{k^2\log n}{b^2\delta^2})$. Thus, the outputs of $\mathsf{BalSBM}$ are the true clusters of $T$, denoted as $T_1,...,T_k$, where $T_i = T\cap V_i$.
\begin{algorithm}[H]
    \caption{$\mathsf{SubCluster}(V,k,\delta,b)$: sub-clusters recovery}
    \label{alg.subcluster}
\begin{algorithmic}[1]
    \STATE Randomly sample $T\subset V$ of size $|T| = \frac{c_0 k^2\log n}{b^2\delta^2}$ for large enough constant $c_0$.
    \STATE Construct $H_T$ using $|T|^2$ queries to the oracle
    \STATE Apply $\mathsf{BalSBM}(H_T,k,\delta, b/2)$ to obtain $X_1,...,X_k$
    \STATE \textbf{output} $X_1,...,X_k$
\end{algorithmic}
\end{algorithm}

\begin{lemma}[\cite{peng21a}]\label{lem:sub-cluster-recovery}
If the input instance $V$ is $b$-balanced, then, with probability $1-|T|^{-7}$, Algorithm~\ref{alg.subcluster} correctly outputs the true clusters $T_1,...,T_k$ of $T$. Moreover, $|T_i|\ge \frac{b|T|}{2k} = \frac{c k\log n}{b\delta^2}$ for some constant $c$ and the sample complexity is $O(\frac{k^4\log^2 n}{b^4\delta^4})$.
\end{lemma}

In the second phase, each $u\in V \setminus T$ identifies its true label. In \cite{peng21a,mazumdar17a}, the algorithm picks $\Theta(\log n/\delta^2)$ distinct vertices from each $T_i$, and query $u$ with them.
If the majority of these answers are $+1$, then add $u$ to $T_i$. The correctness of this method can be proved by applying the Chernoff bound together with a union bound. Note the sample complexity is $O(\frac{nk\log n}{\delta^2})$. Although simple, further improving it seems difficult a priori. Next, we show how to cross this barrier using BAI algorithms.

\subsection{True Cluster Identification via BAI}\label{sec:true-cluster-id-bal}

After the first phase, we get $k$ disjoint sets $\{X_1,...,X_k\}$, each of size $\Omega(\frac{k\log n}{b\delta^2})$; and each $X_i$ is a sub-cluster, i.e., $X_i\subset V_i$. Our key observation is that, for any vertex $u\in V$, identifying the cluster of $u$ can be modeled as a best arm identification problem. To see this, each $X_i$ corresponds to the $i$th arm $a_i$ in BAI. Let $\mathcal{A}$ be an algorithm for $k$-BAI. When $\mathcal{A}$ pulls an arm $a_i$, we pick a fresh vertex $v\in X_i$ and query $(u,v)$. Then $\mathcal{A}$ gets a reward $1$ if the oracle returns $+1$, and a reward $0$ otherwise.
\begin{algorithm}[H]
    \caption{$\mathsf{TrueClusterId}(u,X_1,...,X_k,\delta,\alpha)$}
    \label{alg.bci}
\begin{algorithmic}[1]
    \STATE Simulate $\mathsf{ME}(k,\delta,\alpha)$:
    \REPEAT
    \STATE Receive an arm pull request from $\mathsf{ME}$, say $a_i$
    \STATE Pick a random new vertex $v\in X_i$, and query $(u,v)$ to get the answer $\tilde{\tau}(u,v)$ from the oracle
    \STATE Return reward $r(a_i):=\frac{\tilde{\tau}(u,v)+1}{2}$ to $\mathsf{ME}$
    \UNTIL $\mathsf{ME}$ terminates and outputs an arm $a_{i^*}$
    \STATE \textbf{Output} Index $i^*$
\end{algorithmic}
\end{algorithm}
\begin{lemma}\label{lem:TCI}
Assume $X_i$ is a sub-cluster and $|X_i|\ge \Omega(\frac{k}{\delta^2}\log \frac{1}{\alpha})$ for all $i$. Then for any $u\in V\setminus T$, Algorithm~\ref{alg.bci} correctly identifies the true cluster of $u$ with probability $1-\alpha$ and the sample complexity is $O(\frac{k}{\delta^2}\log \frac{1}{\alpha})$.
\end{lemma}
\begin{proof}
   We simulate $\mathsf{ME}$ on a standard BAI instance, where the best arm corresponds to the true cluster that $u$ belongs to; each arm pull receives an i.i.d.\ reward with mean $\frac{1}{2}+\frac{\delta}{2}$ for the best arm and with mean $\frac{1}{2}-\frac{\delta}{2}$ for the other arms. The simulation works as long as there are enough number of vertices in each $X_i$ to generate i.i.d.\ samples, which is satisfied by our assumption. Then the lemma directly follows from Lemma~\ref{lem:BAI}.
\end{proof}

\subsection{Identify the True Clusters of All Vertices}
Lemma~\ref{lem:TCI} is only for a fixed vertex, and our goal is to identify the true clusters for all vertices simultaneously. An immediate idea is to set $\alpha = 1/n$, so that each of the $n$ vertices find the correct cluster with probability at least $1-n^{-1}$. Then by a union bound, the goal is achieved with constant probability. 
However, the total number of queries is $\Theta(\frac{nk\log n}{\delta^2})$, which is the same as in previous work. 

Instead, we use a recursive strategy to identify true clusters in multiple rounds. In the first round, we apply Algorithm~\ref{alg.bci} on each $u\in V$ with $\alpha=1/4$, and each $u$ gets an index $\tilde{c}(u) \in [k]$. By Lemma~\ref{lem:TCI}, $\tilde{c}(u)$ indicates the true cluster of $u$ with probability at least $3/4$.  
By the Chernoff-Hoeffding bound, except for an exponentially (in $|V|$) small probability, at least $1/2$-fraction of $V$ find their correct clusters. Then, we remove them and repeat the same process on the remaining vertices. Since we don't know which nodes have been labelled correctly, we need a verification step to verify whether $\tilde{c}(u)$ is indeed the correct cluster for each $u$. 

For this step, there is only one candidate cluster to verify, so a standard majority test suffices (Algorithm~\ref{alg.btc}). We pick $\Theta(\frac{\log n}{\delta^2})$ fresh vertices from $T_{\tilde{c}(u)}$ and query $u$ with them. If the majority of the answers are $1$, then we conclude $\tilde{c}(u)$ is the true cluster and remove $u$ in the next round. Otherwise, $u$ is kept for the next round. By the Chernoff-Hoeffding bound, this verification step is correct with probability $1-\frac{1}{n^c}$. By union bound, with probability at least $1-\frac{1}{n^{c-2}}$, all verification steps ever executed are correct, and we assume this event happens. 
\begin{algorithm}[H]
    \caption{$\mathsf{ClusterVerify}(v,B)$}
    \label{alg.btc}
\begin{algorithmic}[1]
    \STATE For each $u\in B$, query $(v,u)$ and get answer $\tilde{\tau}(v,u)$
    \STATE Compute $d(v,B)=\sum_{u\in B}\mathbbm{1}_{\{\tilde{\tau}(v,u)=1\}}$
    \IF{$d(v,B)\ge\frac{1}{2}\abs{B}$}
    \STATE Output \textbf{TRUE} 
    \ELSE
    \STATE Output \textbf{FALSE}
    \ENDIF
\end{algorithmic}
\end{algorithm}

\begin{lemma} \label{thm.btc}
    Let $B$ be a $(\eta,V_i)$-biased set of size at least $\frac{16\log n}{\eta^2\delta^2}$ for some cluster $V_i$, then with probability at least $1-n^{-6}$, the following holds for all $v\in V$ and all clusters $V_i$ simultaneously,\\
        (1) if $v\in V_i$, $\mathsf{ClusterVerify}(v,B)$ returns \textbf{TRUE};\\
        (2) if $v\notin V_i$, $\mathsf{ClusterVerify}(v,B)$ returns \textbf{FALSE}.
\end{lemma}
\begin{proof}
    Lemma \ref{thm.d} and union bound for all $v$ and $i$ complete the proof.
\end{proof}

\subsection{Query Complexity of the Second Phase}\label{sec:BalQueryComplexity}
According to the discussion above, the number of remaining vertices after each round decreases by half, i.e., the number of vertices considered in the $i$th round is at most $\frac{n}{2^{i-1}}$. Since we always set $\alpha = 1/4$, Algorithm~\ref{alg.bci}  incurs $O(\frac{nk}{2^{i-1}\delta^2})$ queries in round $i$ by Lemma~\ref{lem:BAI}. The number of queries for verification is $O(\frac{n\log n}{2^{i-1}\delta^2})$, and therefore there are $O(\frac{n(k+\log n)}{2^{i-1}\delta^2})$ queries in the $i$th round. The total number over all rounds is 
$$\sum_{i\ge 1}O(\frac{n(k+\log n)}{2^{i-1}\delta^2}) = O(\frac{n(k+\log n)}{\delta^2}).$$

\subsection{Put It All Together}
The pseudo code of our algorithm for balanced instances is presented in Algorithm~\ref{alg.bnc}.
\begin{algorithm}[H]
    \caption{$\mathsf{BalNoisyClustering}(V,k,\delta,b)$}
    \label{alg.bnc}
\begin{algorithmic}[1]
    \STATE Let $\alpha=1/4$. Initialize $U=V$ 
    \STATE Invoke $\mathsf{SubCluster}(V,k,\delta,b)$ to obtain sub-clusters $X_1,...,X_k$
    \STATE For each $u\in V$, set $D_u=\set{X_1,...,X_k}$ \COMMENT{Candidate clusters of $u$}
    \STATE For each $X_i$, pick arbitrary $X_i'\subseteq X_i$ of size $\frac{1600\log n}{\delta^2}$
    \WHILE{$\abs{U}\ge\frac{c_0 k^2\log n}{b^2\delta^2}$}
        \FOR{each $v\in U$}
            \STATE  $j=\mathsf{TrueClusterId}(v,D_v,\delta,\alpha)$
            \STATE $o=\mathsf{ClusterVerify}(v,X'_j)$
            \IF{$o=$\textbf{TRUE}}
                \STATE Update $X_j=X_j\cup\set{v}$, $U=U\setminus\set{v}$
            \ELSE
                \STATE Update $D_v = D_v \setminus \{X_j\}$
            \ENDIF
        \ENDFOR
    \ENDWHILE
    \FOR{each $v\in U$}
        \FOR{each $X_j\in D_v$}
            \IF{$\mathsf{ClusterVerify}(v,X'_j)=$ \textbf{TRUE}}
                \STATE  Update $X_j = X_j\cup v$
            \ENDIF
        \ENDFOR
    \ENDFOR
    \OUTPUT $k$ clusters $X_1,...,X_k$
\end{algorithmic}
\end{algorithm}
Now we are ready to prove the main theorem for the nearly balanced instances, i.e. Theorem \ref{thm.bnc}.

\textit{Proof of Theorem \ref{thm.bnc}}. First, it holds with probability $1-o_n(1)$ that $X_1,...,X_k$ are always sub-clusters (i.e., $X_i\subset V_i$) during the execution of Algorithm~\ref{alg.bnc}. To see this, with probability $1-o_n(1)$, it is true after invoking $\mathsf{SubCluster}$ in line 2 (Lemma~\ref{lem:sub-cluster-recovery}). Conditioned on this, $X'_i\subset V_i$ for each $i$ ($X_i'$ is $(1/2,V_i)$-biased). Each time we add a vertex $v$ to $X_i$, it holds that $\mathsf{ClusterVerify}(v,X'_j)=$ \textbf{TRUE}, which means by Lemma \ref{thm.btc}, with probability $1-\frac{1}{n^6}$, $v\in V_i$ and after adding $v$, $X_i$ is still a sub-cluster of $V_i$. The claim follows by applying a union bound over all updates to $X_i$. In the cleanup stage (line 16-22 in Algorithm~\ref{alg.bnc}), all remaining vertices will eventually find their correct clusters with probability at least $1-\frac{1}{n^6}$ (again by Lemma \ref{thm.btc}), therefore the correctness of the algorithm is proved.

For the query complexity, in the first phase $\mathsf{SubCluster}$ incurs $O(\frac{k^4\log^2 n}{b^4\delta^4})$ queries (Lemma~\ref{lem:sub-cluster-recovery}). The query complexity of the second phase is analyzed in subsection~\ref{sec:BalQueryComplexity}, which is  $O(\frac{n(k+\log n)}{\delta^2})$. In the cleanup stage, each of the remaining vertices invokes  $\mathsf{ClusterVerify}$ less than $k$ times, and thus the total number of queries is 
$O(\frac{k^3\log^2 n}{b^2 \delta^4})$. To sum up the query complexiy of Algorithm~\ref{alg.bnc} is at most
$$
O\left(\frac{n(k+\log n)}{\delta^2}+\frac{k^4\log^2 n}{b^4\delta^4}\right).
$$
And the running time of our algorithm is polynomial in $n,k,1/\delta$. \qed

\section{Algorithm for General Instances}\label{sec:generalCase}
For general instances, let $s_1,...,s_k$ be the sizes of $V_1,...,V_k$ respectively. W.l.o.g, we assume $s_1\ge s_2\cdots \ge s_k$. Similar as for balanced instances, we still randomly sample a subset $T$ of size $\poly(k,1/\delta,\log n)$ and query all pairs of vertices in $T$. As before, let $H_T=(T, E_T)$ be the graph with positive answers and define $T_i = T\cap V_i$. Now $T$ might be extremely unbalanced and we cannot apply Lemma~\ref{thm.vu}. 

The idea of Peng \& Zhang \yrcite{peng21a} is to first prune all vertices in $T$ that belong to small clusters. Let $T'$ be the remaining vertices, which is then a balanced instance. Peng \& Zhang \yrcite{peng21a} observe that if $u$ belongs to a small cluster, then its degree in $H_T$ is small, and vice versa. The next lemma provides a formal statement of this intuition, which can be proved by the Chernoff-Hoeffding bound.
\begin{lemma}[\cite{peng21a}]\label{lem:degreeGap}
    Let $b\in[0,1/2], h\in [k]$, and $d_h = \frac{t}{2}-\left(\frac{1}{2}-\frac{1}{k}+\frac{(h+1/2)b}{k^2}\right)\delta t$. Assume $t=|T| = \frac{64 k^4 \log n}{b^2 \delta^2}$ and $u\in T_i$, then it holds with probability $1-\frac{1}{n^7}$: (1) If $|V_i| \le \frac{n}{k} - \frac{b(h+1)n}{k^2}$, then the degree of $u$ in $H_T$ is smaller than $d_h$; (2) If $|V_i|\ge \frac{n}{k} - \frac{bhn}{k^2}$, then the degree of $u$ in $H_T$ is larger than $d_h$.
\end{lemma}
A natural idea is to remove all vertices in $T$ with degree smaller than $d_h$ and let $T'$ be the remaining vertices. We are guaranteed that, if $|V_i| \le \frac{n}{k} - \frac{b(h+1)n}{k^2}$ then $V_i \cap T' = \emptyset$; and if $|V_i| \ge \frac{n}{k} - \frac{bhn}{k^2}$, then $T_i\subset T'$. However, $T'$ can still be unbalanced, since for $V_i$ with size in between, $|V_i\cap T'|$ could be some positive but small integer. Peng \& Zhang \yrcite{peng21a} then prove that if $s_k< bn/k$ (the input is unbalanced), then there must exists $h < k$ such that $s_{h+1} < \frac{n}{k} - \frac{b(h+1)n}{k^2}$ and $s_h\ge \frac{n}{k} - \frac{bhn}{k^2}$. Then, after pruning vertices with degree smaller than $d_h$,\footnote{The only problem is that such $h$ exists but is unknown, which will be discussed later.} $T'$ only contains vertices from $V_1,...,V_h$ and is $\frac{h}{2k}$-balanced, whose clusters can be recovered by applying $\mathsf{BalSBM}$ (Lemma~\ref{thm.vu}) and let $X_1,...,X_h$ be the clusters of $T'$. 

Next, we grow $X_1,...,X_h$ to $V_1,...,V_h$ by identifying the true cluster of each $u \in V\setminus T'$ as before. In \cite{peng21a}, this is achieved by querying $u$ with $\Theta(\frac{\log n}{\delta^2})$ vertices from each $X_i$ and checking whether the majority is positive. After getting $V_1,...,V_h$, the remaining clusters are solved recursively. Again the query complexity is suboptimal and we wish to improve it using BAI.

\subsection{A New Size-gap Lemma}
Similar as in the balanced case, for each $u\in V\setminus T'$, we invoke $\mathsf{TrueClusterId}(u,X_1,...,X_h,\delta,1/4)$ (Algorithm~\ref{alg.bci}) to identify the cluster label of $u$, and then invoke $\mathsf{ClusterVerify}$ (Algorithm~\ref{alg.btc}) to verify the cluster. However, the same argument as for the balanced case does not work any more. Let $n_i$ be the number of remaining vertices in the beginning of the $i$th round. In the balanced case, at least half of the $n_i$ vertices will be settled in the $i$th round. For the unbalanced case, this is not true when $h<k$, since it is possible that the majority of the $n_i$ vertices belongs to $V_{h+1},...,V_k$; such vertices will never be settled in this round. 

To overcome the above obstacle, we first strengthen the size-gap lemma of \cite{peng21a}. We show that for any unbalanced instance, there exists $h<k$ such that there is a gap between $s_h$ and $s_{h+1}$. Moreover, the number of vertices in $V_1,...,V_h$ is $\Theta(n)$. 

\begin{lemma}\label{lem:sizegap}
    If $s_k<\frac{n}{4k}$, then there exists $h<k$ such that
    \begin{align*}
        s_h \ge \frac{n}{2k}-h\cdot\frac{n}{4k^2} \textnormal{ and } 
        s_{h+1} < \frac{n}{2k}-(h+1)\cdot\frac{n}{4k^2}.
    \end{align*}
    Moreover, $\sum_{i\le h} s_i \ge n/2$.
\end{lemma}

Note this is still not sufficient. For example, it is possible that $h = \Theta(k)$ and $\sum_{i>h} s_i = \Theta(n)$; for this case, each $\mathsf{TrueClusterId}$ incurs $\Theta(h/\delta^2) = \Theta(k/\delta^2) $ queries and in each round $\mathsf{TrueClusterId}$ is executed on $\Theta(n)$ vertices, and thus $\Theta(nk/\delta^2)$ queries. Note it needs $\Omega(\log n)$ rounds to ensure all vertices in $V_1,...,V_h$ have been identified, which is required for the success of the algorithm from \cite{peng21a}. We take a different strategy. Given sub-clusters, we do not try to recover all vertices belongs to those clusters, but stop after one round of filtering. By the new size-gap lemma, this removes a constant fraction of vertices. Then we apply the same procedure on the rest of vertices recursively. Details are provided in Section~\ref{sec.merge}. 

\subsection{Recovering Sub-clusters}
Based on the idea outlined above, we describe a subroutine, which outputs $h$ sub-clusters assuming $h$ is given. We first randomly sample a set $T\subset V$ of size $t=\frac{ck^4\log n}{\delta^2}$, and then query all pairs $u,v\in T$ and construct the graph $H_T=(T,E_T)$.
Next, we remove small-degree vertices and apply $\mathsf{BalSBM}$ (Lemma~\ref{thm.vu}) on the remaining graph 
(see Algorithm~\ref{alg:GapSBM}).

\begin{algorithm}[H]
    \caption{$\mathsf{GapSBM}(T,h,\delta)$}
    \label{alg:GapSBM}
\begin{algorithmic}[1]
    \STATE Remove all vertices in $H_T$ with degree less than $d_h:= (\frac{1}{2}-\frac{\delta}{2})t+\delta\cdot(\frac{1}{2k}-\frac{h+1/2}{4k^2})t$
    \STATE Let $T_h'$ be the set of remaining vertices and $H_{T_h'}$ be the graph induced by $T_h'$ 
    \STATE Invoke $\mathsf{BalSBM}(H_{T_h'},h ,\delta,\frac{h}{4k})$ to obtain clusters $X_1,...,X_h$
    \STATE \textbf{output} $X_1,...,X_h$
\end{algorithmic}
\end{algorithm}
The algorithm is essentially the same as in \cite{peng21a}, however, equipped with our new size-gap lemma, the guarantee is stronger.
\begin{lemma}\label{lem:gapSBM}
    Let $T$ be a random subset of $V$ of size $t$. With probability $1-O(k^{-24}\log^{-8}n)$, we have\\
    (1) If $h$ is the index from Lemma~\ref{lem:sizegap}. Then, the output of $\mathsf{GapSBM}$ are $h$ sub-clusters of $V_1,...,V_h$ whose sizes satisfy $s_i\ge t/4k$ and $\sum_{i\le h} s_i \ge n/2$. \\
    (2) For an arbitrary $h\in [k]$, all vertices in $T$ that belongs to $V_i$ with $s_i\ge \frac{n}{2k}-h\cdot\frac{n}{4k^2}$ must be contained in the output.
\end{lemma}
The proof follows the analysis of \cite{peng21a}. For completeness, we include a proof in Appendix~\ref{app:generalCase}.

Before giving the full algorithm, there is another issue to deal with: $h$ from Lemma~\ref{lem:sizegap} is unknown. We simply apply $\mathsf{GapSBM}(T, h, \delta)$ for all $h\in [k]$ in decreasing order, and then perform a test on the output to check whether all subsets returned are indeed sub-clusters. This idea was successfully implemented in \cite{peng21a}. Our algorithm $\mathsf{GapSubcluster}$ (Algorithm~\ref{alg:GapSubcluster}) follows the same idea and the proof of the following lemma is given in Appendix~\ref{app:generalCase}. Our proof is a non-trivial extension due to part (c). 
\begin{lemma}\label{lem:gapCluster-unknown-h}
    Let $T$ be a random subset of $V$ of size $t=\frac{ck^4\log n}{\delta^2}$. $\mathsf{GapSubcluster}(T, k, \delta)$ outputs $X_1,...,X_h$ for some $h\in [k]$ or \textbf{Fail}. If $s_k <|V|/4k $, then with probability $1-O(k^{-24}\log^{-8}n)$, it does not output \textbf{Fail}. If the algorithm does not \textbf{fail} then
    \vspace{-0.2cm}
    \begin{enumerate}
        \item[(a)] $|X_i| \ge \frac{t}{4k}$.
        \item[(b)] $X_i$ is $(0.1, V_{g(i)})$-biased for some $g(i) \in [k]$.
        \item[(c)] $\sum_{j: g^{-1}(j)\neq \emptyset} |V_j| \ge n/5$.
    \end{enumerate}
        \vspace{-0.2cm}
    The query complexity is $O(t^2)$.
\end{lemma}

\subsection{True Cluster Identification: One Round}\label{sec:gap-cluster-identification-1round}
By Lemma~\ref{lem:gapCluster-unknown-h}, $\mathsf{GapSubcluster}$ outputs $X_1,...,X_h$ such that $X_i$ is $(0.1, C)$-biased for some cluster $C$ in $V_1,..., V_k$. Similar as in Section~\ref{sec:true-cluster-id-bal}, we use $\mathsf{TrueClusterId}(u, X_1,..., X_h, \delta, 1/4)$ to identify the cluster of $u$. Here the situation is slightly different, since now each $X_i$ is not a strict sub-cluster. But the same algorithm still works, except that we need to use a smaller $\delta$, denoted as $\delta'$ to be determined later. To simulate an arm $a_i$ for $i\in[h]$, we randomly sample a vertex $v$ from $X_i$ and query $(u,v)$, and the reward is $r(a_i) = \frac{\tilde{\tau}(u,v)+1}{2}$. Suppose $u\in V_{g(i)}$, the expected reward of pulling $a_i$ is at least 
$$(\frac{1}{2}+0.1)(\frac{1}{2}+\frac{\delta}{2}) + (\frac{1}{2}-0.1) (\frac{1}{2}-\frac{\delta}{2})=  \frac{1}{2} + \frac{\delta}{10}.$$
Similarly, if $u\notin V_{g(i)}$, the reward is at most $\frac{1}{2}-\frac{\delta}{10}$.
So, with probability $3/4$, the true cluster of $u$ is identified by $\mathsf{ME}(h,\frac{\delta}{5},1/4)$, and the query complexity of which is the same up to a constant.
One subtlety is that the random vertex $v$ may have been sampled before, and the resulting rewards may not satisfy i.i.d.\ condition. Note, if $X_i$ is a strict sub-cluster, we could just pick any $v$ that has not been queried before. However, the correctness of BAI on such a reward distribution still holds up to a constant blowup in the query complexity by a more careful concentration argument.
\begin{lemma} \label{lem:true-cluster-id-biased}
    Let $X_1,...,X_h$ be the subsets obtained in Lemma~\ref{lem:gapCluster-unknown-h}. if $u$ belongs to $V_{g(i^*)}$ for some $i^*$, then $\mathsf{TrueClusterId}(u, X_1,..., X_h, \delta/5, \alpha)$ returns $i^*$ with probability $1-\alpha$.
\end{lemma}

Similarly, we use $\mathsf{ClusterVerify}$ to verify whether the cluster label for each $u$ is correct. By (c) of Lemma~\ref{lem:gapCluster-unknown-h}, at least a constant fraction of $V$ will pass this verification and are added to the correct clusters. The query complexity of this part is still $O(\frac{|V|k}{\delta^2}+\frac{|V|\log |V|}{\delta^2})$.

\subsection{Iterative Clustering}\label{sec.merge}
The above true cluster identification step successfully labels a constant fraction of vertices. Let $U$ be the set of vertices remain unlabelled. To make $U$ a valid noisy clustering instance, we also exclude $T$ from $U$, since all pairs from $T$ have been queried. Then, we apply the same procedure on $U$. After at most $O(\log n)$ rounds, the size of $U$ becomes too small, and then we do a cleanup step. 
\paragraph{Merging clusters of different rounds.} 
Note that each round may produce as many as $k$ sub-clusters and there are totally $O(k\log n)$ over all rounds, denoted as $C_i,...,C_{\ell}$. So we need to merge sub-clusters from the same $V_i$. This is another noisy clustering problem, but easier. Given two sub-clusters $C_i$ and $C_j$, to test whether both belong to the same cluster, we just pick one vertex $u$ from $C_i$ and $m=O(\frac{\log n}{\delta^2})$ vertices $v_1,...,v_m$ from $C_j$, and query $(u,v_i)$ for all $i\in[m]$. If the majority of the answer is positive, then $C_i,C_j$ are merged. We do this testing for all $k^2\log^2 n$ pairs of clusters. By the Chernoff-Hoeffding bound and union bound, with probability at least $1-n^{-6}$, the merging result is correct. The sample complexity of the merging process is at most $O(\frac{k^2\log^3 n}{\delta^2})$.

\subsection{The Final Algorithm}
\begin{algorithm}[H]
    \caption{$\mathsf{NoisyClustering}(V,k,\delta)$}
    \label{alg.nc}
\begin{algorithmic}[1]
    \STATE Let $c$ be a sufficiently large constant. Initialize $R=\emptyset$, $C_i=\emptyset$ for $i\in[k]$, $r=1$, $U=V$.
    \WHILE{$\abs{U}\ge\frac{c k^4\log n}{\delta^2}$}
      \STATE Randomly sample $T^{r}\subset U$ of size $|T^r| = \frac{c k^4\log n}{\delta^2}$
        \STATE Inkove $\mathsf{GapSubcluster}(T^{r},k,\delta)$ to obtain sub-clusters $X_1,...,X_h$. 
        \IF{$\mathsf{GapSubcluster}$ output \textbf{Fail}}
            \STATE Invoke $\mathsf{BalNoisyClustering}(U, k, \delta, 1/4)$ to get $k$ clusters $X_1,...,X_k$
            \STATE $U=\emptyset$
            \STATE Merge $X_1,...,X_k$ into $C_i$'s as described in \ref{sec.merge}
            \STATE \textbf{break}
        \ENDIF
        \STATE For each $i$, pick a subset $X_i'\subseteq X_i$ of size $\frac{c \log n}{\delta^2}$
        \STATE $U=  U \setminus T^r$, $R = R\cup T^r$, $Y_i=\emptyset$
        \FOR{each $v\in U$}
            \STATE $j=\mathsf{TrueClusterId}(v, X_1,...,X_h, \delta, 1/8)$ 
             \STATE $o=\mathsf{ClusterVerify}(v, X'_j)$
            \IF{$o=$\textbf{TRUE}}
                \STATE Update $Y_j=Y_j\cup\set{v}$, $U=U\setminus\set{v}$
            \ENDIF
        \ENDFOR
        \STATE Merge $Y_1,...,Y_k$ into $C_i$'s as described in \ref{sec.merge}
        \STATE $r= r+1$
    \ENDWHILE
  \STATE  \COMMENT{Cleanup stage}
    \FOR{each $C_i$} 
        \STATE Pick a subset $C_i'\subseteq C_i$ of size $\frac{16\log n}{\delta^2}$
        \STATE Let $\tilde{C}_i=\set{v\in U\cup R:\mathsf{ClusterVerify}(v,C'_i) = \textbf{TRUE}}$
        \STATE Update $C_i = C_i\cup \tilde{C}_i$
    \ENDFOR
    \OUTPUT all the clusters $C_i$'s
\end{algorithmic}
\end{algorithm}
Now the outline of our algorithm for general case is complete, and we present the pseudo code of the final algorithm in Algorithm~\ref{alg.nc}. The proof of the main theorem, i.e. Theorem \ref{thm.nc}, is provided in Appendix~\ref{app:GenearlTheorem}.

\section{Lower bound}\label{sec.lb}
In \cite{mazumdar17a}, they proved a $\Omega(\frac{nk}{\delta^2})$ lower bound on the number of queries. Here we prove that $\Omega(\frac{n\log n}{\delta^2})$ is also a lower bound and thus our algorithm is optimal.

\paragraph{Best arm identification with failure.} Before proving this lower bound, we define a variant of the best $2$-arm identification problem. We call this problem \emph{best arm identification with failure}, denoted as $(\delta,\alpha)$-BAIF. Given $2$ arms of expected reward $\frac{1}{2}+\frac{\delta}{2}$ and $\frac{1}{2}-\frac{\delta}{2}$ respectively for some $\delta\in(0,1)$, suppose the rewards are binary i.e. $\set{0,1}$. We still want to find the optimal arm, but now we allow the algorithm to fail with small probability. More specifically, the algorithm can return FAIL (not any one of the $2$ arms) with probability at most $1/8$, or return an arm otherwise (NOT FAIL). When the algorithm does return an arm, then it should be the optimal one with probability at least $1-\alpha$ (conditioned on NOT FAIL).

We claim that any algorithm that solves $(\delta,\alpha)$-BAIF, has a sample complexity lower bound $\Omega(\frac{1}{\delta^2}\log\frac{1}{\alpha})$, which is the same as the well known lower bound for the original best arm identification problem \cite{mannor04a}. The proof of Lemma \ref{thm.banditlb} is in Appendix \ref{appendix.banditlb}.

\begin{lemma} \label{thm.banditlb}
    Any (randomized) algorithm that solves $(\delta,\alpha)$-BAIF has $\Omega(\frac{1}{\delta^2}\log\frac{1}{\alpha})$ sample complexity.
\end{lemma}

\paragraph{Reduction from BAIF to clustering with noisy queries.}
Now we show that BAIF can be reduced to the problem of identifying the cluster of one vertex in the faulty oracle model with $k=2$ clusters. Suppose $\mathcal{A}$ is an algorithm for clustering with noisy queries with success probability at least $\frac{7}{8}$, then we can derive an algorithm $\mathcal{A}_i'$ indexed by a vertex $i$ that solves $(\delta, 1/n)$-BAIF. The query complexity of $\mathcal{A}_i'$ is the number of queries in $\mathcal{A}$ that involve vertex $i$. Lemma \ref{thm.banditlb} implies that $\mathcal{A}$ must query $i$ at least $\Omega(\frac{1}{\delta^2}\log n)$ times. So if we can show that there are $\Omega(n)$ such $i$'s, the query complexity of $\mathcal{A}$ is $\Omega(\frac{n}{\delta^2}\log n)$.

Consider a graph distribution $\mathcal{D}$ of $n$ vertices, which samples a random graph as follows. First let each vertex belongs to one of the two clusters $C_0,C_1$ randomly with equal probability. Then fix the clusters and sample a SBM-type graph, i.e. assign an edge to each pair of vertices from the same cluster with probability $\frac{1}{2}+\frac{\delta}{2}$, and assign an edge to each pair of vertices from different clusters with probability $\frac{1}{2}-\frac{\delta}{2}$. 

Let $\mathcal{A}$ be any algorithm that solves clustering with noisy queries with probability at least $\frac{7}{8}$. We next provide an algorithm for $(\delta, 1/n)$-BAIF (Algorithm~\ref{alg.reduction}). 

\begin{algorithm}[H]
    \caption{$\mathcal{A}'_s$: algorithm for BAIF}
    \label{alg.reduction}
\begin{algorithmic}[1]
\REQUIRE two arms $a_0,a_1$, suppose $a_{opt}$ is the best arm
    \STATE Assign vertices $v_1,...,v_{s-1},v_{s+1},...,v_{n}$ to the cluster $C_0$ or $C_1$ randomly with equal probability, let $c(i)$ be the cluster of $v_i$.
    \STATE For the special vertex $v_s$, its cluster is indicated by the (unknown) index of the best arm, i.e. $opt$.\\
$\#\# $ Simulate $\mathcal{A}$ on the instance constructed above:
    \REPEAT
    \STATE Receive an query request from $\mathcal{A}$, say $(v_i, v_{i'})$
    \IF{$i,i' \neq s$}
        \IF{$v_i, v_{i'}$ are in the same cluster}
            \STATE $\beta \sim \mathsf{Bernoulli}(\frac{1}{2}+\frac{\delta}{2})$
        \ELSE
            \STATE $\beta \sim \mathsf{Bernoulli}(\frac{1}{2}-\frac{\delta}{2})$
        \ENDIF
        \STATE answer $\tilde{\tau}(v_i, v_{i'}) = 2\beta -1$
    \ELSIF{$i=s$ (or $i'=s$)}
    \STATE pull arm $a_{c(i')}$ (or $a_{c(i)}$) get a reward $r\in\set{0,1}$
    \STATE answer $\tilde{\tau}(v_i, v_{i'}) = 2r-1$
    \ENDIF
    \UNTIL $\mathcal{A}$ terminates and outputs the cluster labels $\tilde{c}(1),...,\tilde{c}(n)$
    \IF{$c(i) = \tilde{c}(i)$ for $i=1,...,s-1$}
        \STATE output $\tilde{c}(s)$
    \ELSE
    \STATE output FAIL
    \ENDIF
\end{algorithmic}
\end{algorithm}
Let $\mathcal{E}$ denote the event that $\mathcal{A}$ correctly recovers all the clusters for an instance sampled from $\mathcal{D}$, and $\mathcal{E}_i,i\in[n]$ denote the event that the algorithm correctly identifies the cluster for vertex $v_i$, i.e., $\tilde{c}(i) = c(i)$. Note $\Pr_{\mathcal{D,A}}[\mathcal{E}] \ge 7/8$. To simplify notation, we omit the subscripts, but the probability is always over both $\mathcal{D}$ and the randomness of $\mathcal{A}$.
\begin{definition}
We call index $s\in [n]$ good, if the conditional probability $\Pr[\mathcal{E}_s|\mathcal{E}_1,...,\mathcal{E}_{s-1}] \ge 1-1/n$.
\end{definition}
\begin{lemma}\label{lem:ngoodindex}
    There are at least $3n/4$ good indices.
\end{lemma}
\begin{proof}
Suppose there are less than $3n/4$ good indices, then, by the chain rule,
\begin{align*}
\Pr[\mathcal{E}] & = \prod_{i=1}^{n}\Pr[\mathcal{E}_i|\mathcal{E}_1,...,\mathcal{E}_{i-1}] 
< (1-\frac{1}{n})^{\frac{n}{4}}\le e^{-\frac{1}{4}} < \frac{7}{8},
\end{align*}
which is a contradiction to the correctness of $\mathcal{A}$. So there must be at least $3n/4$ good indices in $[n]$.
\end{proof}

\begin{lemma}\label{lem:goodBAIF}
    Suppose $s$ is good, then $\mathcal{A}'_s$ solves $(\delta, 1/n)$-BAIF. The sample complexity of $\mathcal{A}'_s$ is the same as the number of queries used in $\mathcal{A}$ that involve $s$. 
\end{lemma}
\begin{proof}
Let us assume $a_0$ and $a_1$ both have probability $1/2$ being the better arm. This is not a restriction, since we can always randomly permute the two arms in the beginning. Then, the clustering instance for $\mathcal{A}$ constructed in $\mathcal{A}'_s$ has the same distribution as $\mathcal{D}$, since each $v_i$, including $v_s$, belongs to $C_0$ or $C_1$ with equal probability and edges are also sampled accordingly. We observe that the cluster label of $v_s$ in the graph is the same as the index of the best arm. 

The event that $\mathcal{A}'_s$ outputs FAIL implies that $\mathcal{A}$ is wrong for at least one vertex, which happens with probability at most $1/8$. On the other hand, if $\mathcal{A}'_s$ does not output FAIL, then $\mathcal{E}_1,...,\mathcal{E}_{i-1}$ all happens. By the assumption that $s$ is good, i.e., $\Pr[\mathcal{E}_s|\mathcal{E}_1,...,\mathcal{E}_{s-1}] \ge 1-1/n$, we conclude that $\tilde{c}(s)$ is the correct cluster label of $v_s$, and hence the index of the best arm, with probability at least $1-1/n$. 

From the description of Algorithm~\ref{alg.reduction}, each arm pull in $\mathcal{A}'_s$ corresponds to a query of $\mathcal{A}$ that involves $s$, which finishes the proof the lemma.
\end{proof}
\vspace{-0.2cm}
Now we are ready to prove the main theorem.
\vspace{-0.2cm}
\begin{proof}[Proof of Theorem \ref{thm.lb}.]
We have shown that, given any algorithm for the noisy clustering problem, we can construct $n$ algorithms $\mathcal{A}'_s$, $s=1,...,n$, for BAIF. Let $S$ be the sample complexity of $\mathcal{A}$, and $S_i$ be the number of queries that involve $i$. We have
\begin{align}\label{eqn:expectedQuery}
    \mathbb{E}_{\mathcal{D,A}} [S] = \frac{1}{2}\sum_{i=1}^n \mathbb{E}_{\mathcal{D,A}} [S_i] \ge \frac{1}{2}\sum_{i: i \textnormal{ is good}} \mathbb{E}_{\mathcal{D,A}} [S_i].
\end{align}
For a good $i$, by Lemma~\ref{lem:goodBAIF}, the algorithm $\mathcal{A}'_s$ solves $(\delta, 1/n)$-BAIF and $\mathbb{E}_{\mathcal{D,A}} [S_i]$ is the expected query complexity of $\mathcal{A}'_s$. Then, by Lemma~\ref{thm.banditlb}, $\mathbb{E}_{\mathcal{D,A}} [S_i] = \Omega(\frac{1}{\delta^2}\log n)$. Since there are at least $3n/4$ good indices (Lemma~\ref{lem:ngoodindex}), then by \eqref{eqn:expectedQuery}, the expected number of query complexity of $\mathcal{A}$ is $\mathbb{E}_{\mathcal{D,A}} [S] =\Omega(\frac{n\log n}{\delta^2})$.
\end{proof}

\section{Conclusion}
In this paper, we study clustering with a faulty oracle, a.k.a. clustering with noisy queries, where the algorithm is allowed to ask a noisy oracle whether two items belong to the same cluster or not. We provide a new algorithm with improved query complexity and also prove a new information-theoretic lower bound that matches our upper bound for a wide range of parameters. 
Along the way, we introduce several new techniques. For the upper bound, the most important new ingredient is to utilize multi-armed bandit algorithms to identify true clusters of vertices. Our lower bound proof also relies on a new connection to the best arm identification with failure problem. We believe our new techniques and insights are of independent interests and might be useful for other problems in noisy query models.

\section*{Acknowledgements}
This work was partly supported by the National Key Research and Development Program of China  (2020AAA0107600).

\bibliography{example_paper}
\bibliographystyle{icml2022}

\newpage
\appendix
\onecolumn

\section{Preliminaries}
We will use the following well known inequalities. See, e.g., Proposition 2.5 in \cite{wainwright19a} and Theorem 2.5.5 in \cite{durrett19a}.

\begin{theorem}[Chernoff-Hoeffding bound] \label{thm.chb}
    Let $n\ge 1$. Suppose there are $n$ independent random variables $X_1,...,X_n$ that are distributed in $[0,1]$. Let $X:=\sum_{i=1}^{n}$. Then for all $\lambda>0$,
    $$
    \mathbb{P}\left(X>\mathbb{E}[X]+\lambda\right),\mathbb{P}\left(X<\mathbb{E}[X]-\lambda\right) \le e^{-2\lambda^2/n}.
    $$
\end{theorem}

\begin{theorem}[Kolmogorov's inequality]\label{thm.kol}
    Let $n\ge 1$. Suppose there are $n$ independent random variables $X_1,...,X_n$ such that $\mathbb{E}[X_i]=0$ and $\var(X_i)<\infty$. Let $S_k=\sum_{i=1}^{k}X_i,k=1,...,n$. Then for all $\lambda>0$,
    $$
    \mathbb{P}\left(\max_{1\le k\le n}|S_k|\ge\lambda\right)\le\lambda^{-2}\var(S_n).
    $$
\end{theorem}

\section{Deferred Proofs from Section~\ref{sec:balancedCase}} \label{app:balancedCase}
Define function $d:V\times 2^V$,
    $$
    d(v,X)=\sum_{u\in X}\mathbbm{1}_{\{\tilde{\tau}(v,u)=+1\}}
    $$
to be the number of positive edges between $v$ and vertices in set $X$, where $\mathbbm{1}$ is the indicator function. This quantity has the following property.

\begin{lemma} \label{thm.d}
    Let $B$ be a $(\eta,C)$-biased set of size at least $\frac{16\log n}{\eta^2\delta^2}$, for some cluster $C$. Then with probability at least $1-n^{-8}$, for any vertex $v\in V$,
    \begin{align*}
        d(v,B) &\ge \left(\frac{1}{2}+\frac{1}{2}\eta\delta\right)|B|, \text{ if } v\in B, \\
        d(v,B) &\le \left(\frac{1}{2}-\frac{1}{2}\eta\delta\right)|B|, \text{ if } v\notin B.
    \end{align*}
\end{lemma}

\begin{proof}
    Let $C_v$ denote the true cluster $v$ belongs to. Let $B_v=B\cap C_v$. 
    \begin{align*}
        \mathbb{E}[d(v,B)] &= \left(\frac{1}{2}+\frac{1}{2}\delta\right)\abs{B_v}+\left(\frac{1}{2}-\frac{1}{2}\delta\right)\abs{B\setminus B_v} \\
        &= \left(\frac{1}{2}-\frac{1}{2}\delta\right)\abs{B}+\delta\abs{B_v}.
    \end{align*}
    Let $\lambda=\frac{1}{2}\eta\delta\abs{B}$, then $\lambda^2/\abs{B}4\log n$. Since $B$ is $(\eta,C)$-biased, then if $v\in C$, $\abs{B_v}\ge(\frac{1}{2}+\delta)\abs{B}$, so we have
    $$
    \mathbb{E}[d(v,B)] \ge \left(\frac{1}{2}+\eta\delta\right)\abs{B}.
    $$
    By Chernoff-Hoeffding bound (see Lemma \ref{thm.chb}), with probability at least $1-\exp{-2\lambda^2/\abs{B}}\ge 1-n^{-8}$,
    $$
    d(v,B) \ge \left(\frac{1}{2}+\frac{1}{2}\eta\delta\right)\abs{B}.
    $$
    Similarly if $v\notin C$, then $\abs{B_v}\le(\frac{1}{2}+\delta)\abs{B}$, so with probability at least $1-\exp{-2\lambda^2/\abs{B}}\ge 1-n^{-8}$,
    \begin{align*}
        d(v,B) &\le \left(\frac{1}{2}-\frac{1}{2}\eta\delta\right)\abs{B}. \qedhere
    \end{align*}
\end{proof}

\section{Deferred proofs from Section~\ref{sec:generalCase}} \label{app:generalCase}
\subsection{Proof of Lemma~\ref{lem:sizegap}}
\begin{proof}
    Let $f(i)=s_i - \left(\frac{n}{2k}-i\cdot\frac{n}{4k^2}\right)$, then 
    \begin{align*}
        f(1) &= s_1 - \left(\frac{n}{2k}-\frac{n}{4k^2}\right) \ge \frac{n}{k}-\frac{n}{2k}+\frac{n}{4k^2} > 0, \\
        f(k) &= s_k - \left(\frac{n}{2k}-k\cdot\frac{n}{4k^2}\right) < \frac{n}{4k}-\frac{n}{4k} \le 0.
    \end{align*}
    Let $h$ be the largest index such that $f(h)\ge 0$. By the above two inequalities, $h\ge 1$ and $h\le k-1$. Then $s_h \ge\frac{n}{2k}-h\cdot\frac{n}{4k^2} $ and $s_{h+1} < \frac{n}{2k}-(h+1)\cdot\frac{n}{4k^2}$. Since $s_{i} \le \frac{n}{2k}$, $\sum_{i>h} \le n/2$, and thus the second part is proved.
\end{proof}

\subsection{Proof of Lemma~\ref{lem:gapSBM}}
\begin{proof}
    Let $\lambda_1 = \frac{t}{16k^2}$, then $\lambda_1^2/t\ge 4\log n$. For any $i\le h$, by Chernoff-Hoeffding bound, with probability at least $1-n^{-8}$, 
    $$
    |T_i| \ge \left(\frac{1}{2k}-\frac{h+1/4}{4k^2}\right)t,
    $$
    For any $i>h$, with probability at least $1-n^{-8}$, 
    $$
    |T_i| < \left(\frac{1}{2k}-\frac{h+3/4}{4k^2}\right)t.
    $$
    In the following, we assume the above condition holds for all $i\in[k]$, which occurs with probability at least $1-n^{-7}$ by the union bound.

    Let $\lambda_2 = \frac{t\delta}{16k^2}$, then $\lambda_2^2/t\ge 4\log n$. For any $i\le h$, for any vertex $v\in T_i$, by Chernoff-Hoeffding bound, with probability at least $1-n^{-8}$, its degree in $H_T$ 
    $$
    d(v,T) \ge \left(\frac{1}{2}-\frac{\delta}{2}\right)t+\delta\cdot\left(\frac{1}{2k}-\frac{h+1/2}{4k^2}\right)t.
    $$
    For any $i>h$, for any vertex $v\in T_i$, with probability at least $1-n^{-8}$, its degree in $H_T$ 
    $$
    d(v,T) < \left(\frac{1}{2}-\frac{\delta}{2}\right)t+\delta\cdot\left(\frac{1}{2k}-\frac{h+1/2}{4k^2}\right)t.
    $$
    Let $d_h = (\frac{1}{2}-\frac{\delta}{2})t+\delta\cdot(\frac{1}{2k}-\frac{h+1/2}{4k^2})t$, then by union bound, with probability at least $1-n^{-7}$, all vertices in $T_1,...,T_h$ have degree at least $d_h$ and all vertices in $T_{h+1},...,T_k$ have vertices less than $d_h$. Then by the description of the algorithm, $T' = \bigcup_{i\le h}T_h$. 

    Now we consider $H_{T'}$, and the number of clusters in $H_{T'}$ is $h$. For each $i\le h$, 
    $$
    |T_i| \ge \left(\frac{1}{2k}-\frac{h+1/4}{4k^2}\right)t \ge \frac{t}{4k} \ge \frac{|T'|}{4k} = \frac{h}{4k}\frac{|T'|}{h}.
    $$
    Thus let $b''=\frac{h}{4k}$, then the partition $T_1,...,T_h$ of $T'$ is $b''$-balanced, and $|T'|\ge \frac{ht}{4k}\ge \frac{256c_0k^3\log n}{\delta^2}$, $\log|T'|\le\log n$. Thus,
    $$
    \frac{|T'|}{\log |T'|} \ge 16kc_0\frac{16k^2}{h^2}\frac{h^2}{\delta^2} \ge \frac{c_0h^2}{(b'')^2\delta^2}.
    $$
    Therefore, by Lemma~\ref{thm.vu}, $\mathsf{BalSBM}(H_{T'},k,\delta,\frac{h}{4k})$ can recover sub-clusters $T_1,...,T_h$ with probability at least $1-|T'|^{-8}$.
\end{proof}

\subsection{Proof of Lemma~\ref{lem:gapCluster-unknown-h}}
\begin{proof}
We first provide a subroutines $\mathsf{TestBias}$ from \cite{peng21a}.

\begin{algorithm}[H]
    \caption{$\mathsf{TestBias}(n,B,\eta)$: test if $B$ is $(\eta,C)$-biased for some cluster $C$}
\begin{algorithmic}[1]
    \FOR{$i=1,...,\frac{16k\log n}{b}$}
        \STATE Randomly sample a vertex $v_i\in V$ and query all pairs $v_i,u$ for $u\in B$
        \IF{$d(v,B)\ge(\frac{1}{2}+\frac{1}{2}\eta\delta)\abs{B}$}
           \STATE \textbf{output} \emph{Yes}
        \ENDIF
    \ENDFOR
    \STATE \textbf{output} \emph{No}
\end{algorithmic}
\end{algorithm}

\begin{lemma}[\cite{peng21a}] \label{thm.tb}
    Let $B$ be a vertex set of size at least $\frac{64\log n}{\eta^2\delta^2}$. $\mathsf{TestBias}(n,B,\eta)$ makes $O(\frac{k\log n|B|}{b})$ queries and satisfies, with probability at least $1-n^{-7}$,\\
    (1) return Yes, if $B$ is $(\eta,C)$-biased for some cluster $C$ of size at least $\frac{bn}{k}$, i.e., $\abs{B\cap C}\ge(\frac{1}{2}+\eta)\abs{B}$;\\
    (2) return No, if $B$ is not $(\frac{\eta}{4},C)$-biased for any cluster $C$, i.e. $\abs{B\cap C}\le(\frac{1}{2}+\frac{\eta}{4})\abs{B}$.
\end{lemma}

\begin{algorithm}[H]
    \caption{$\mathsf{GapSubcluster}(T,k,\delta)$: find a good index $h$ and the corresponding sub-clusters}
    \label{alg:GapSubcluster}
\begin{algorithmic}[1]
    \FOR{$h=k-1,...,1$}
            \STATE Invoke $\mathsf{GapSBM}(T,h,\delta)$ to get $h$ clusters $X_1,...,X_h$
        \IF{$|X_i|< \frac{t}{4k}$ for some $i\le h$}
            \STATE \textbf{continue}
        \ENDIF
        \IF{for all $i\in[h]$, $\mathsf{TestBias}(n,X_i,0.1)$ returns \emph{Yes}}
            \STATE \textbf{output} $X_1,...,X_h$
        \ENDIF
    \ENDFOR
    \STATE \textbf{output} \textbf{Fail}
\end{algorithmic}
\end{algorithm}
By Lemma~\ref{lem:sizegap}, there is some $h^*\le k-1$ satisfying the size gap condition, and thus by (1) of Lemma~\ref{lem:gapSBM}, with probability $1-O(k^{-24}\log^{-8}n)$, at least one execution of $\mathsf{GapSBM}$ in the for loop will output $h^*$ sub-clusters of $V_1,...,V_{h^*}$ whose sizes satisfies $s_i\ge t/4k$ and $\sum_{i\le h^*} s_i \ge n/2$. In particular, all these sub-clusters will pass $\mathsf{TestBias}$. Therefore, if $s_k < n/4k$, $\mathsf{GapSubcluster}$ will not output \textbf{Fail} with probability $1-O(k^{-24}\log^{-8}n)$. 

On the other hand, if $\mathsf{GapSubcluster}$ succeeds, then the size of each $X_i$ in the output at least $t/4k$, and thus (a) is proved. Moreover, each $X_i$ is $(0.1, V_{g(i)})$-biased for some $g(i) \in [k]$ since they passed $\mathsf{TestBias}(n,X_i,\eta)$, which implies (b). 

To prove (c), let $h$ be final index when $\mathsf{GapSubcluster}$ successfully terminates and let $t_i = |T\cap V_i|$. By (2) of Lemma~\ref{lem:gapSBM}, all vertices in $T$ that belongs to $V_i$ with $s_i\ge \frac{n}{2k}-h\cdot\frac{n}{4k^2}$ must be contained in the output of $\mathsf{GapSBM}$. Let $\ell$ be the largest number such that $s_{\ell}\ge \frac{n}{2k}-h\cdot\frac{n}{4k^2}$. In other words, for all $i\le \ell$, $|(\bigcup_{i=1}^h X_i)\cap V_i| = t_i$. By Chernoff bound, with probability $1-\frac{1}{n^5}$, $t_i\le 1.01 \frac{ts_i}{n}$ holds for all $i\le \ell$. Let $x_i = |X_i|$. Note, when $h\neq h^*$, we have no control of the output of $\mathsf{GapSBM}(T,h,\delta)$; it could even be possible that $g(i)=g(j)$ for some $i,j\le h$ and $i\neq j$. Since $X_i$ is $(0.1, V_{g(i)})$-biased, which means at least $0.6$ fraction of $X_i$ are from $ V_{g(i)}$, then
\begin{align*}
0.6 \sum_{j\in g^{-1}(i)} x_j\le t_i \le 1.01 \frac{ts_i}{n}.
\end{align*}
Then sum the above inequality over all $i$ such that $g^{-1}(i) \neq \emptyset$, we have
\begin{align}\label{eqn:si-lowerbound}
\sum_{i: g^{-1}(i) \neq \emptyset} s_i\ge 0.5 \sum_{j=1}^h x_j \cdot \frac{n}{t}.
\end{align}
We emphasize that $\sum_{j=1}^h x_j \neq t$, since we have pruned vertices from $T$ with low degrees. However, we can still show that $\sum_{j=1}^h x_j = \Theta(t)$. To see this, first observe that $h\ge h^*$ since we try all $h$ in a decreasing order. Then $d_h\le d_{h^*}$ and $T'_{h^*} \subset T'_h$. By Lemma~\ref{lem:gapSBM}, all vertices from $V_1,...,V_{h^*}$ are preserved in $T'_{h^*}$, i.e., $|T'_{h^*}| \ge \sum_{i=1}^{h^*} t_i$. By Chernoff bound, $t_i\ge 0.99 \frac{t s_i}{n}$ and By Lemma~\ref{lem:sizegap}, $\sum_{i=1}^{h^*} s_i \ge n/2$. As a result, 
\begin{align*}
    \sum_{j=1}^h x_j = |T'_h| \ge |T'_{h^*}| \ge \sum_{i=1}^{h^*} t_i \ge 0.99 \frac{t}{n}\sum_{i=1}^{h^*}s_i \ge 0.4 t. 
\end{align*}
Together with \eqref{eqn:si-lowerbound}, we finish the proof of (c).

For the sample complexity, $\mathsf{GapSBM}$ makes $t^2$ queries. Each $\mathsf{TestBias}(n,X_i,0.1)$ makes $O\left( {k|X_i| \log n} \right) \le O(kt\log n)$. There are at most $k^2$ calls of $\mathsf{TestBias}(n,X_i,0.1)$, so the total complexity of this part is $O(k^3 t \log n) \le O(t^2)$.
\end{proof}

\subsection{Proof of Theorem \ref{thm.nc}.} \label{app:GenearlTheorem}
Let $U_1,...,U_k$ be the clusters of $U$ at the beginning of the $r$-th round, and $s_i = |U_i|$. W.l.o.g., assume $s_1\ge s_2\cdots\ge s_k$.

\paragraph{Case 1: $s_k< |U|/4k$.}
By Lemma~\ref{lem:gapCluster-unknown-h}, then with probability $1-o_n(1)$, $\mathsf{GapSubcluster}(T^{r},k,\delta)$ does not \textbf{Fail}. Conditioned on this, each $X_i$ is $(0.1, U_{g(i)})$-biased and $|X_i|\ge |T^r|/4k$ (Lemma \ref{lem:gapCluster-unknown-h}). For each $v\in U \setminus T^r$ and $v\in U_{g(i)}$ for some $i$, it holds that $j=i$ (line 13) with probability $7/8$ (Lemma~\ref{lem:true-cluster-id-biased}). Then, with probability $1-1/n^6$ $\mathsf{ClusterVerify}(v, X'_j)$ output true and $v$ is added to $X_i$. On the other hand, if $v\notin U_{g(i)}$ for all $i$, $v$ remains in $U$ with probability $1-1/n^6$. For $s_k< |U|/4k$, we say the round succeeds if $\mathsf{GapSubcluster}(T^{r},k,\delta)$ does not \textbf{Fail}, and for all $v$ $\mathsf{TrueClusterId}(v, X_1,...,X_h, \delta, 1/8)$ and $\mathsf{ClusterVerify}(v, X'_j)$ output correctly. By union bound, the round succeeds with probability $1-o_n(1)$. For a successful round, all $v\in U \setminus T^r$ such that $v\in U_{g(i)}$ for some $i$ are added to $Y_i$ and all other vertices $v$ in $U$ remain in $U$ for the next round. By (3) of Lemma~\ref{lem:gapCluster-unknown-h}, a successful round reduces the number of vertices in $U$ by a factor of $4/5$. So there are at most $4\log n$ rounds. 

\paragraph{Case 2: $s_k\ge |U|/4k$.} Note $\mathsf{GapSubcluster}(T^{r},k,\delta)$ may or may not \textbf{Fail} in this case. Suppose, it outputs \textbf{Fail}, then $\mathsf{BalNoisyClustering}(U, k, \delta, 1/4)$ is called. Now $U$ is $1/4$-balanced, by Theorem~\ref{thm.bnc}, $X_1,...,X_k$ are the true clusters of $U$ with probability $1-o_n(1)$ (the round succeeds); then $U=\emptyset$ and directly goes to the cleanup stage. $\mathsf{GapSubcluster}$ may not fail and still output $X_1,...,X_h$, which satisfies the guarantees in Lemma~\ref{lem:gapCluster-unknown-h}. If this happens, then everything is the same as in Case 1, and the round succeeds with probability $1-o_n(1)$. 

\paragraph{Case 3: $\abs{U}< \frac{c k^4\log n}{\delta^2}$.} In this case, the algorithm does nothing in this round and goes to the cleanup stage. To simplify the analysis, we assume that the algorithm always runs for exactly $4\log n$ rounds by adding dummy rounds after $U$ becomes too small. A dummy round is considered successful.

\paragraph{Error probability and unlabelled vertices.} For $r\in [4\log n]$, let $E_r$ denote the event that the $r$-round succeeds and the merging result is correct at the end of the round. We have from the analyses above and in Section~\ref{sec.merge}, $\Pr[E_r]\ge 1-o_n(1)$. Since this lower bound holds for any configuration at the beginning of the round, $\Pr[E_r~|~E_1,...,E_{r-1}] \ge 1-o_n(1)$ for all $r$ (since $E_1,...,E_{r-1}$ only affects the initial configuration of the $r$-th round). Then, by the chain rule, we have
\begin{align*}
\Pr[E_1,...,E_{4\log n}] \ge (1-o_n(1))^{4\log n},
\end{align*}
which is still $1-o_n(1)$ since the $o_n(1)$ term above is smaller than $k^{-28}\log^{-8} n$. 

Next we assume $E_1,...,E_{4\log n}$ happen. Before the cleanup stage, we get $C_1,...,C_{k}$, and assume $|C_1|\ge\cdots \ge |C_k|$. It is possible $C_i = \emptyset$, and let $i$ be the largest index such that $|C_i| > 0$. Note, conditioned on $E_1,...,E_{4\log n}$, each nonempty $C_i$ has size at least $t/4k = \frac{c k^3\log n}{4\delta^2}$. After the cleanup stage, each nonempty $C_i$ will grow to the full cluster $V_j$ for some $j\in [k]$. However, there could be some vertices remain unidentified. In the worst case, all vertices in $U\cup R$ are unidentified. Since $|U|<\frac{c k^4\log n}{\delta^2}$ and $|R| \le 4\log n \cdot \frac{c k^4\log n}{\delta^2} = \frac{4c k^4\log^2 n}{\delta^2}$, we have $|U\cup R|\le \frac{5c k^4\log^2 n}{\delta^2}$. Since each $V_i$ is either fully recovered or $V_i\subset U\cup R$, all $V_i$ with $|V_i|\ge \frac{5c k^4\log^2 n}{\delta^2}$ must have been recovered. We remark that algorithm by \cite{peng21a} recovers all the clusters of size $\Omega(\frac{k^4\log n}{\delta^2})$ and algorithm by \cite{mazumdar17a} recovers all the clusters of size $\Omega(\frac{k\log n}{\delta^4})$. By running their algorithms on $U\cup R$, we achieve the same guarantees on the size of recovered clusters.

\paragraph{Query Complexity.} Now we bound the query complexity of the final algorithm. Recall that there are at most $4\log n$ rounds. In each round $r$, $\mathsf{GapSubcluster}$ makes $t^2 = \frac{k^8\log^2 n}{\delta^4}$ queries. Similar as in the balanced case, $\mathsf{TrueClusterId}$ and $\mathsf{ClusterVerify}$ make $O(\frac{|U|(k+\log n)}{\delta^2})$ queries. In the merging step, at most $2k$ sub-clusters are merged, which incurs $O(\frac{k^2\log n}{\delta^2})$ queries. Since $|U|$ decreases exponentially, over all rounds, the total query complexity is
$$
O\left(\frac{n(k+\log n)}{\delta^2}+\frac{k^{8}\log^3 n}{\delta^4}\right).
$$
This finish the proof of Theorem~\ref{thm.nc}.

\section{Deferred Proofs from Section~\ref{sec.lb}} \label{appendix.banditlb}
\begin{lemma}[Lemma \ref{thm.banditlb} restated]
    There exist positive constant $c_1,c_2,\delta_0$ and $\alpha_0$, such that for every $\delta\in(0,\delta_0)$ and $\alpha\in(0,\alpha_0)$ and for every algorithm that returns FAIL with probability at most $\frac{1}{8}$ and returns a $\delta$-optimal arm with probability at least $1-\alpha$ conditioned on NOT FAIL, there exists some problem instances such that the expected number of trials satisfies
    $$
    \mathbb{E}[T]\ge c_1\frac{1}{\delta^2}\log\frac{c_2}{\alpha}.
    $$
    In particular, $\delta_0$ and $\alpha_0$ can be taken equal to $1/8$ and $e^{-8}/4$
\end{lemma}

\begin{proof}
    The structure of this proof is similar to the proof of Theorem 1 in \cite{mannor04a}. Let us consider the following two problem instances that we will refer to as "hypotheses" denoted by $H_0$ and $H_1$ respectively.
    \begin{align*}
        H_0 &: \mu_0=\frac{1}{2}+\frac{\delta}{2},\quad \mu_1=\frac{1}{2}-\frac{\delta}{2}, \\
        H_1 &: \mu_0=\frac{1}{2}-\frac{\delta}{2},\quad \mu_1=\frac{1}{2}+\frac{\delta}{2}.
    \end{align*}
    So under $H_0$, arm $a_0$ is optimal, while under $H_1$ arm $a_1$ is optimal.
    
    Let $\delta_0=1/8$ and $\alpha_0=e^{-8}/4$, fix some $\delta\in(0,\delta_0)$, $\alpha\in(0,\alpha_0)$ and any algorithm that satisfies the condition in the lemma. We denote by $\mathbb{E}_l$ and $\mathbb{P}_l$ the expectation and probability under hypothesis $H_l$, respectively. Let $T_i$ be the number of trails that involves arm $a_i$, so $T=T_0+T_1$. We define $t^*$ by
    $$
    t^*=\frac{1}{c\delta^2}\log\frac{1}{4\alpha}=\frac{1}{c\delta^2}\log\frac{1}{\theta},
    $$
    where $\theta=4\alpha$ and $c$ is an universal constant that we will specify its value later. Throughout the proof, we use $E^C$ to denote the complement of event $E$.
    
    Assume that $\mathbb{E}_0[T]\le t^*$, and we claim that under this assumption, the probability that the algorithm returns arm $a_0$ under $H_1$ conditioned on NOT FAIL will exceeds $\alpha$, thus violates the condition in the lemma.
    
    We define some events $A$ and $C_0,C_1$ where various random variables do not deviate significantly from their expected values. Define
    $$
    A=\set{T\le 4t^*}.
    $$
    By Markov's inequality,
    $$
    \mathbb{P}_0(A^C)=\mathbb{P}_0(T>4t^*)\le\frac{\mathbb{E}_0[T]}{4t^*}\le\frac{t^*}{4t^*}=\frac{1}{4}.
    $$
    It follows that
    $$
    \mathbb{P}_0(A)\ge\frac{3}{4}.
    $$
    
    Define $K_i^{t}=\sum_{s=1}^{t}r_i^s$ to be the total reward if arm $a_i$ is tried $t$ (not necessarily consecutive) times, where $r_i^s$ is the reward obtained at the $s$-th time arm $a_i$ is tried. Define
    \begin{align*}
        C_0 &= \set{\max_{1\le t\le 4t^*}\abs{K_0^t-(\frac{1}{2}+\frac{\delta}{2})t}<\sqrt{t^*\log(\frac{1}{\theta})}}, \\
        C_1 &= \set{\max_{1\le t\le 4t^*}\abs{K_1^t-(\frac{1}{2}-\frac{\delta}{2})t}<\sqrt{t^*\log(\frac{1}{\theta})}},
    \end{align*}
    and $C=C_0\cap C_1$.

    \begin{lemma}[Generalized Lemma 2 in \cite{mannor04a}]
        We have $\mathbb{P}_0(C)>\frac{3}{4}$ for $i=0,1$.
    \end{lemma}
    
    \begin{proof}
        We first prove a more general result. Suppose the expected reward of arm $a_i$ is $\mu_i^{(l)}$ under $H_l$, let
        $$
        C_i^{(l)}=\set{\max_{1\le t\le 4t^*}\abs{K_i^t-\mu_i^{(l)}t}<\sqrt{t^*\log(\frac{1}{\theta})}}.
        $$
        Note that $K_i^t-\mu_i^{(l)}t$ is a $\mathbb{P}_l$-martingale. By Kolmogorov's inequality,
        $$
        \mathbb{P}_l\left((C_i^{(l)})^C\right) = \mathbb{P}_l\left(\max_{1\le t\le 4t^*}\abs{K_i^t-\mu_i^{(l)}t}\ge\sqrt{t^*\log(\frac{1}{\theta})}\right) \le \frac{\mathbb{E}_l[(K_i^{4t^*}-4t^*\mu_i^{(l)})^2]}{t^*\log(\frac{1}{\theta})}.
        $$
        Since $\mathbb{E}_l[K_i^{4t^*}]=4t^*\mu_i^{(l)}$, so $\mathbb{E}_l[(K_i^{4t^*}-4t^*\mu_i^{(l)})^2]=4t^*\mu_i^{(l)}(1-\mu_i^{(l)})$. Note that $\theta<e^{-8}$ then we have
        $$
        \mathbb{P}_l(C_i)\ge 1-\frac{4t^*\mu_i^{(l)}(1-\mu_i^{(l)})}{t^*\log\frac{1}{\theta}} \ge \frac{7}{8},
        $$
        where the last inequality follows because $\mu_i^{(l)}(1-\mu_i^{(l)})\le \frac{1}{4}$. Let $l=0$, then $\mathbb{P}_0(C_i)\ge\frac{7}{8}$ for $i=0,1$, and by a simple union bound, $\mathbb{P}_0(C)\ge\frac{3}{4}$.
    \end{proof}
    
    \begin{lemma}\label{thm.simpled}
        If $0\le x\le \frac{1}{\sqrt{2}}$ and $y\ge 0$, then
        $$
        (1-x)^y\ge e^{-dxy},
        $$
        where $d=1.78$.
    \end{lemma}
    
    \begin{proof}
        Simple calculation shows that $\log(1-x)+dx\ge 0$ for any $x\in[0,\frac{1}{\sqrt{2}}]$, rearranging terms and taking $y\ge 0$ gives the proof.
    \end{proof}
    
    Let $D$ be the event that the algorithm returns FAIL. Let $I$ be the arm that the algorithm returns if NOT FAIL. Let $B=\set{I=a_0}$, i.e. the algorithm returns arm $a_0$. Recall the condition in the lemma, we have $\mathbb{P}_0(D)\le\frac{1}{8}$, $\mathbb{P}_0(B|D^{C})\ge 1-\alpha$, $\alpha<e^{-8}/4<\frac{1}{8}$, so $\mathbb{P}_0(B)\ge(1-\mathbb{P}_0(D))(1-\alpha)>\frac{3}{4}$. Let $S=A\cap B\cap C$ be the event that $A,B,C$ occur simultaneously. Since $\mathbb{P}_0(A)\ge\frac{3}{4}$, $\mathbb{P}_0(B)\ge\frac{3}{4}$, $\mathbb{P}_0(C)>\frac{3}{4}$, we then have $\mathbb{P}_0(S)>\frac{1}{4}$.
    
    Now we prove that if $c\ge 400$ and $\mathbb{E}_0[T]\le t^*$, then $\mathbb{P}_1(B)>\delta$.
    
    Let $W$ be the history of the process (the sequence of arms selected at each time and corresponding rewards) until the algorithm terminates. Define the likelihood function $L_l$ by
    $$
    L_l(w) = \mathbb{P}_l(W=w),
    $$
    for every possible history $w$. Note that the choices of the algorithm only depend on the history, and are independent of which hypothesis is true. Let $K_0:=K_0^{T_0}$, $K_1:=K_1^{T_1}$. The likelihood ratio is given by
    \begin{align*}
        \frac{L_1(W)}{L_0(W)} &= \frac{(\frac{1}{2}+\frac{\delta}{2})^{K_1+T_0-K_0}(\frac{1}{2}-\frac{\delta}{2})^{K_0+T_1-K_1}}{(\frac{1}{2}+\frac{\delta}{2})^{K_0+T_1-K_1}(\frac{1}{2}-\frac{\delta}{2})^{K_1+T_0-K_0}} \\
        &=\left(\frac{1-\delta}{1+\delta}\right)^{2(K_0-\frac{1+\delta}{2}T_0)-2(K_1-\frac{1-\delta}{2}T_1)+\delta(T_0+T_1)} \\
        &\ge (1-2\delta)^{2\abs{K_0-\frac{1+\delta}{2}T_0}+2\abs{K_1-\frac{1-\delta}{2}T_1}+\delta T},
    \end{align*}
    where the last inequality is because $1>\frac{1-\delta}{1+\delta}\ge 1-2\delta$. We now lower bound the right-hand side above when $S$ occurs.
    
    Since $S=A\cap B\cap C$ occurs, we have the followings
    \begin{align*}
        T &\le 4t^* = \frac{4}{c\delta^2}\log(\frac{1}{\theta}), \\
        \abs{K_1-\frac{1-\delta}{2}T_1} &\le 2\sqrt{t^*\log(\frac{1}{\theta})} = \frac{2}{\sqrt{c}\delta}\log(\frac{1}{\theta}), \\
        \abs{K_0-\frac{1+\delta}{2}T_0} &\le 2\sqrt{t^*\log(\frac{1}{\theta})} = \frac{2}{\sqrt{c}\delta}\log(\frac{1}{\theta}).
    \end{align*}
    Thus,
    \begin{align*}
        \frac{L_1(W)}{L_0(W)} &\ge (1-2\delta)^{2\cdot \frac{2}{\sqrt{c}\delta}\log(\frac{1}{\theta}) + \delta\cdot \frac{4}{c\delta^2}\log(\frac{1}{\theta})} \\
        &= (1-2\delta)^{(\frac{4}{\sqrt{c}}+\frac{4}{c})\frac{1}{\delta}\log(\frac{1}{\theta})} \\
        &\ge e^{-d\cdot 2\delta\cdot(\frac{4}{\sqrt{c}}+\frac{4}{c})\frac{1}{\delta}\log(\frac{1}{\theta})} \\
        &= \theta^{\frac{8d}{\sqrt{c}}+\frac{8d}{c}},
    \end{align*}
    where the last inequality is by Lemma \ref{thm.simpled}.
    
    By picking sufficiently large $c$ ($c=400$ suffices), we obtain that $\frac{8d}{\sqrt{c}}+\frac{8d}{c}\le 1$, thus $\frac{L_1(W)}{L_0(W)}\ge\theta=4\alpha$. Note that this holds for every possible history $W=w$ whenever $S$ occurs, i.e. more precisely,
    $$
    \frac{L_1(W)}{L_0(W)}\mathbbm{1}_{S} \ge 4\alpha\mathbbm{1}_{S},
    $$
    where $\mathbbm{1}_{S}$ is the indicator function of event $S$. Then
    $$
    \mathbb{P}_1(B)\ge\mathbb{P}_1(S)=\mathbb{E}_1[\mathbbm{1}_{S}]=\mathbb{E}_0\left[\frac{L_1(W)}{L_0(W)}\mathbbm{1}_{S}\right]\ge\mathbb{E}_0[4\alpha\mathbbm{1}_{S}]=4\alpha\mathbb{P}_0(S)>\alpha,
    $$
    where the last inequality is because $\mathbb{P}_0(S)>\frac{1}{4}$.
    
    To summarize, we have shown that when $c\ge 400$, if $\mathbb{E}_0[T]\le \frac{1}{c\delta^2}\log(\frac{1}{4\alpha})$, then $\mathbb{P}_1(B|D^{C})\ge\mathbb{P}_1(B)>\alpha$, i.e. the algorithm will return arm $a_0$ with probability larger than $\alpha$ under $H_1$, which is not a $\delta$-optimal arm. So for every algorithm that satisfies the condition in the lemma, we must have $\mathbb{E}_0[T]>\frac{1}{c\delta^2}{\log(\frac{1}{4\alpha})}$
\end{proof}


\end{document}